\documentclass[a4paper,11pt]{article}

\pdfoutput=1

\usepackage[margin=1in]{geometry}

\usepackage[
  bookmarks=true,
  bookmarksnumbered=true,
  bookmarksopen=true,
  pdfborder={0 0 0},
  breaklinks=true,
  colorlinks=true,
  linkcolor=black,
  citecolor=black,
  filecolor=black,
  urlcolor=black,
]{hyperref}
\usepackage{isomath}
\usepackage[round]{natbib}
\usepackage{amsmath,amsfonts,amssymb,amsthm}
\usepackage[capitalize]{cleveref}
\usepackage{nicefrac}

\theoremstyle{plain}
\newtheorem{theorem}{Theorem}[section]
\newtheorem{lemma}[theorem]{Lemma}
\newtheorem{corollary}[theorem]{Corollary}
\newtheorem{prop}[theorem]{Proposition}
\newtheorem{claim}[theorem]{Claim}

\theoremstyle{definition}
\newtheorem{example}[theorem]{Example}

\DeclareMathOperator*{\E}{\mathbb{E}}
\newcommand{\Ex}{\E}
\providecommand{\F}{\mathcal{F}}
\providecommand{\eps}{\epsilon}
\newcommand{\eqdef}{\triangleq}

\newcommand{\R}{\mathbb{R}}
\newcommand{\N}{\mathbb{N}}

\title{Submultiplicative Glivenko-Cantelli and\texorpdfstring{\\}{ }Uniform Convergence of Revenues}

 \author{
Noga Alon\thanks{Tel Aviv University, Israel, and Microsoft Research. \texttt{nogaa@tau.ac.il}.}
\and Moshe Babaioff\thanks{Microsoft Research. \texttt{moshe@microsoft.com}.}
\and Yannai A. Gonczarowski\thanks{The Hebrew University of Jerusalem, Israel, and Microsoft Research. \texttt{yannai@gonch.name}.}
\and Yishay Mansour\thanks{Tel Aviv University, Israel, and Google Research, Israel. \texttt{mansour@tau.ac.il}.
The research was done while author was co-affiliated with Microsoft Research.}
\and Shay Moran\thanks{Institute for Advanced Study, Princeton. \texttt{shaymoran1@gmail.com}.
Part of the research was done while author was co-affiliated with Microsoft Research.}
\and  Amir Yehudayoff\thanks{Technion --- Israel Institute of Technology, Israel. \texttt{amir.yehudayoff@gmail.com}.}
}

\date{November 4, 2017}

\begin{document}

\maketitle

\begin{abstract}
In this work we derive a variant of the classic Glivenko-Cantelli Theorem,
which asserts uniform convergence of the empirical Cumulative Distribution Function (CDF)
to the CDF of the underlying distribution. Our variant allows for tighter convergence bounds for extreme values of the CDF.

We apply our bound in the context of \emph{revenue learning},
which is a well-studied problem in economics and algorithmic game theory.
We derive sample-complexity bounds
on the uniform convergence rate of the empirical revenues to the true revenues,
assuming a bound on the $k$th moment of the valuations, for any (possibly fractional) $k>1$.

For uniform convergence in the limit, we give a complete characterization and a zero-one~law:
if the first moment of the valuations is finite, then uniform convergence almost surely occurs; conversely,
if the first moment is infinite, then uniform convergence almost never occurs.
\end{abstract}

\section{Introduction}
\label{sec:intro}

A basic task in machine learning is to learn an unknown distribution $\mu$,
given access to samples from it.
A natural and widely studied criterion for learning a distribution
is approximating its Cumulative Distribution Function (CDF).
The seminal Glivenko-Cantelli Theorem~\citep{glivenko1933,cantelli1933}
addresses this question when the distribution $\mu$ is over the real numbers.
It determines the behavior of the empirical distribution function
as the number of samples grows:
let $X_1,X_2,\ldots$ be a sequence of i.i.d.\ random variables
drawn from a distribution $\mu$ on $\R$
with Cumulative Distribution Function (CDF) $F$,
and let $x_1, x_2, \ldots$ be their realizations.
The \emph{empirical distribution} $\mu_n$ is
\[\mu_n \eqdef \frac{1}{n} \sum_{i=1}^n \delta_{x_i},\]
where $\delta_{x_i}$ is the constant distribution supported on $x_i$.
Let $F_n$ denote the CDF of $\mu_n$, i.e., $F_n(t)\eqdef\frac{1}{n}\cdot\bigl|\{1\leq i\leq n : x_i \leq t\}\bigr|$.
The Glivenko-Cantelli Theorem formalizes
the statement that $\mu_n$ converges to $\mu$ as $n$ grows,
by establishing that $F_n(t)$ converges to $F(t)$, uniformly over all $t\in\mathbb{R}$:

\begin{theorem}[Glivenko-Cantelli Theorem, \citeyear{glivenko1933}]
\label{t:GCthm}
Almost surely,
\[\lim_{n \to \infty} \sup_t \bigl| F_n(t) - F(t)\bigr| =  0.\]
\end{theorem}

Some twenty years after \citeauthor{glivenko1933} and \citeauthor{cantelli1933} discovered this theorem, \citeauthor*{dkw1956} (DKW)
strengthened this result by giving an almost\footnote{The inequality due to \cite{dkw1956} has a larger constant $C$ in front of the exponent on the right hand side.} tight quantitative
bound on the convergence rate. In \citeyear{massart1990}, \citeauthor{massart1990}
proved a tight inequality, confirming a conjecture due
to \cite{birnbaummccarty1958}:
\begin{theorem}[\citealp{massart1990}]
\label{t:massart}
$\Pr\Bigl[ \sup_t \bigl| F_n (t) - F(t)\bigr|  > \eps \Bigr]\leq 2\exp(-2n\eps^2)$
for all $\eps >0$, $n\in\mathbb{N}$.
\end{theorem}

The above theorems show that, with high probability, $F$ and $F_n$
are close up to some \emph{additive} error.
We would have liked to prove a stronger, \emph{multiplicative}
bound on the error:
\[\forall t : \bigl|F(t) - F_n(t)\bigr| \leq \eps \cdot F(t).\]
However, for some distributions, the above event has probability $0$, no matter how large $n$ is.
For example, assume that $\mu$ satisfies $F(t)> 0$ for all $t$.
Since the empirical measure $\mu_n$ has finite support,
there is $t$ with $F_n(t) = 0$; for such a value of $t$, such a multiplicative
approximation fails to hold.

So, the above multiplicative requirement is too strong to hold
in general.
A natural compromise is to consider a \emph{submultiplicative}
bound:
\[\forall t : \bigl| F(t) - F_n(t) \bigr| \leq \eps\cdot F(t)^\alpha,\]
where $0 \leq \alpha < 1$.
When $\alpha=0$, this is the additive bound studied in the context
of the Glivenko-Cantelli Theorem.
When $\alpha = 1$, this is the unattainable multiplicative bound.
Our first main result shows that
the case of $\alpha < 1$ is attainable:

\begin{theorem}[Submultiplicative Glivenko-Cantelli Theorem]\label{t:qual}
Let $\eps > 0$, $\delta > 0$ and $0 \leq \alpha < 1$.
There exists $n_0(\eps,\delta,\alpha)$ such that for all $n>n_0$,
with probability $1-\delta$:
\[ \forall t  : \bigl\lvert F(t) - F_n(t)\bigr\rvert \leq \eps\cdot F(t)^\alpha.\]
\end{theorem}

It is worth pointing out a central difference between Theorem~\ref{t:qual}  and other generalizations
of the Glivenko-Cantelli Theorem: for example, the seminal work of \cite{zbMATH03391742}
shows that for every class of events $\F$ of VC dimension $d$,
there is $n_0=n_0(\eps,\delta,d)$ such that for every $n\geq n_0$, with probability
$1-\delta$ it holds that $\forall A\in\F  : \bigl\lvert p(A) - p_n(A)\bigr\rvert \leq \eps$.
This yields Glivenko-Cantelli by plugging $\F=\bigl\{(-\infty, t] : t\in\R\bigr\}$, which has VC dimension $1$.
In contrast, the submultiplicative bound from Theorem~\ref{t:qual} 
does not even extend to the VC dimension $1$ class $\F=\bigl\{\{t\} : t\in\R\bigr\}$.
Indeed, pick any distribution $p$ over $\R$ such that $p\bigl(\{t\}\bigr)=0$ for every $t$,
and observe that for every sample $x_1,\ldots, x_n$, it holds that
$p_n\bigl(\{x_i\}\bigr)\geq \nicefrac{1}{n}$, however $p\bigl(\{x_i\}\bigr)=0$, and therefore, as long as $\alpha > 0$, 
it is never the case that $\Bigl\lvert p\bigl(\{x_i\}\bigr) - p_n\bigl(\{x_i\}\bigr) \Bigr\rvert \leq p\bigl(\{x_i\}\bigr)^\alpha$.
Theorem~\ref{t:qual} is proven in Section~\ref{sec:GC},
which also includes other extensions.

Our second main result gives an explicit upper bound on
$n_0(\eps,\delta,\alpha)$:
\begin{theorem}[Submultiplicative Glivenko-Cantelli Bound]\label{thm:samplecomplexity}
Let $\eps,\delta\leq\nicefrac{1}{4}$, and $\alpha < 1$. Then
\[n_0(\eps,\delta,\alpha)\leq
\max\left\{\frac{ \ln\bigl(\nicefrac{6}{\delta}\bigr)}{2\eps^2}\Bigl(\frac{\eps\delta}{3}\Bigr)^{-\frac{4\alpha}{1-\alpha}},~~~
(D+1)\Biggl(10\cdot\ln\left(12\cdot\frac{D+4}{\delta(1-\alpha)}\right)\Biggr)^{\frac{4\alpha}{1-\alpha}}\right\}, \]
where
$D=\frac{\ln(\nicefrac{6}{\delta})}{2\eps^2}
\Bigl(\frac{\eps\delta}{6}\cdot
{\ln\bigl(\frac{1+\alpha}{2\alpha}\bigr)}\Bigr)^{-\frac{4\alpha}{1-\alpha}}$.
\end{theorem}
Note that for fixed $\eps,\delta$, when $\alpha\to 0$ the
above bound approaches the familiar
$O\Bigl(\frac{\ln(\nicefrac{1}{\delta})}{\eps^2}\Bigr)$ bound by DKW
and \citeauthor{massart1990} for $\alpha=0$. On the other hand, when $\alpha\to 1$
the above bound tends to $\infty$, reflecting the
fact that the multiplicative variant of
Glivenko-Cantelli ($\alpha=1$) does not hold.
Theorem~\ref{thm:samplecomplexity} is proven in Appendix~\ref{app:thm:rev}.

Note that the dependency of the above bound on the confidence parameter $\delta$ is polynomial.
This contrasts with standard uniform convergence rates, 
which, due to applications of concentration bounds such as Chernoff/Hoeffding, 
achieve logarithmic dependencies on $\delta$.
These concentration bounds are not applicable in our setting when the CDF values are very small,
and we use Markov's inequality instead.
The following example shows that a polynomial dependency on $\delta$ 
is indeed necessary and is not due to a limitation of our proof.

\begin{example}
For large $n$, consider $n$ independent samples $x_1,\ldots,x_n$ from the uniform distribution over~$[0,1]$,
and set $\alpha = \nicefrac{1}{2}$ and $\eps = 1$.
The probability of the event
\[\exists i :x_i\leq1/n^3\]
is roughly $1/n^2$: indeed, the complementary event has probability 
$(1-1/n^3)^n \approx \exp(-1/n^2) \approx 1-1/n^2$.
When this happens, we have:
$F_n(1/n^3) \geq 1/n >\!> 1/n^3 + 1/n^{\nicefrac{3}{2}} 
= F(1/n^3) + \bigl[F(1/n^3)\bigr]^{\nicefrac{1}{2}}$.
Note that this happens with probability inverse polynomial in $n$ 
(roughly $1/n^2$) and not inverse exponential.
\end{example}

\begin{sloppypar}
\paragraph{An application to revenue learning.}
We demonstrate an application of our Submultiplicative
Glivenko-Cantelli Theorem in the context of a widely studied problem
in economics and algorithmic game theory: the problem of revenue
learning. In the setting of this problem, a seller has to decide
which price to post for a good she wishes to sell. Assume that each
consumer draws her private valuation for the good from an unknown
distribution $\mu$. We envision that a consumer with valuation~$v$
will buy the good at any price $p\leq v$, but not at any higher
price. This implies that the expected revenue at price $p$ is simply
$r(p)\eqdef p\cdot q(p)$, where $q(p)\eqdef\Pr_{V\sim\mu}[V\geq p]$.
\end{sloppypar}

In the language of machine learning, this problem can be phrased as
follows: the examples domain $Z\eqdef \mathbb{R}^+$ is the set of
all valuations $v$. The hypothesis space $H\eqdef \mathbb{R}^+$ is
the set of all prices $p$. The revenue (which is a gain, rather than
loss) of a price $p$ on a valuation $v$ is the function $p\cdot
1_{\{p\leq v\}}$.

The well-known \emph{revenue maximization} problem is to find a price $p^*$
that maximizes the expected revenue,
given a sample of valuations drawn i.i.d.\ from $\mu$.
In this paper, we consider the more demanding \emph{revenue
estimation} problem: the problem of well-approximating $r(p)$, simultaneously for all prices~$p$,
from a given sample of valuations. (This clearly also implies
a good estimation of the maximum revenue and of a price that yields it.)
More specifically, we address the following question:
when do the \emph{empirical revenues}, $r_n(p)\eqdef p\cdot q_n(p)$, where $q_n(p)\eqdef\Pr_{V\sim\mu_n}[V\geq p]=\frac{1}{n}\cdot\bigl|\{1\leq i\leq n : x_i \geq t\}\bigr|$,
uniformly converge to the true revenues $r(p)$?
More specifically, we would like to show that for some $n_0$, for $n\geq n_0$
we have with probability $1-\delta$ that
\[
\bigl| r(p) - r_n(p) \bigr| \leq \epsilon.
\]
The revenue estimation problem is a basic instance of the more general
problem of uniform convergence of empirical estimates.
The main challenge in this instance is that the prices are unbounded (and so are the
private valuations that are drawn from the distribution $\mu$).

Unfortunately, there is no (upper) bound on $n_0$ that is only a function of
$\epsilon$ and $\delta$. Moreover, even if we add the expectation of
valuations, i.e.,  $\Ex[V]$ where $V$ is distributed according to
$\mu$, still there is no bound on $n_0$ that is a function of only those three
parameters (see Section~\ref{sec:UC1MOM} for an example).
In contrast, when we consider higher moments of the distribution
$\mu$, we are able to derive bounds on the value of $n_0$.
These bounds are based on our Submultiplicative Glivenko-Cantelli
Bound.
Specifically, assume that $\Ex_{V\sim \mu}[V^{1+\theta}]\leq C$ for
some $\theta>0$ and $C\geq 1$. Then, we show that for any
$\eps,\delta\in (0,1)$, we have
\[\Pr \Bigl[\exists v :~\bigl\lvert r(v) - r_n(v)\bigr\rvert > \eps\Bigr] \leq
\Pr \biggl[\exists v :~\bigl\lvert q(v) - q_n(v)\bigr\rvert >
\frac{\eps}{C^{\frac{1}{1+\theta}}}q(v)^{\frac{1}{1+\theta}}\biggr].\]
This essentially reduces uniform convergence bounds to our
Submultiplicative Glivenko-Cantelli variant. It then follows that there exists
$n_0(C,\theta,\epsilon,\delta)$ such that for any $n\geq n_0$, with
probability at least~$1-\delta$,
\[
\forall v:\;\;\; \bigl|r_n(v)-r(v)\bigr|\leq \eps.
\]
We remark that when $\theta$ is large, our bound yields
$n_0\approx O\Bigl(\frac{\ln(\nicefrac{1}{\delta})}{\eps^2}\Bigr)$, which
recovers the standard sample complexity bounds obtainable via DKW and \citeauthor{massart1990}.

When $\theta\to 0$, our bound diverges to infinity,
reflecting the fact (discussed above)
that there is no bound on $n_0$ that depends only
on $\eps,\delta$, and $\Ex[V]$.
Nevertheless, we find that $\Ex[V]$
qualitatively determines whether uniform convergence occurs
in the limit. Namely, we show that
\begin{itemize}
\item
If $\Ex_\mu[V] < \infty$, then almost surely
$\lim_{n\to\infty}\sup_{v}\bigl\lvert r(v) - r_n(v)\bigr\rvert=0$,
\item
Conversely, if $\Ex_\mu[V] = \infty$, then almost never
$\lim_{n\to\infty}\sup_{v}\bigl\lvert r(v) - r_n(v)\bigr\rvert =0$.
\end{itemize}

\subsection{Related work}

\paragraph{Generalizations of Glivenko-Cantelli.}
Various generalizations of the Glivenko-Cantelli Theorem were established.
These include uniform convergence bounds for more general classes of functions as well as
more general loss functions (for example, \citealp{zbMATH03391742,VapnikBook,KolPan2000,DBLP:journals/jmlr/BartlettM02}).
The results that concern unbounded loss functions are most relevant to this work 
(for example, \citealp{CortesMM10,CortesGM13,VapnikBook}).
We next briefly discuss the relevant results from \cite{CortesGM13} in the context of this paper; 
more specifically, in the context of Theorem~\ref{t:qual}.
To ease presentation, set $\alpha$ in this theorem to be $\nicefrac{1}{2}$.
Theorem~\ref{t:qual} analyzes the event where the empirical quantile
is bounded by\footnote{For consistency with the canonical statement of the Glivenko-Cantelli theorem, we stated our submultiplicative variants of this theorem with regard to the CDFs $F_n$ and $F$. However, these results also hold when replacing these CDFs with the respective quantiles (tail CDFs) $q_n$ and $q$. See Section~\ref{sec:rev-qual} for details.}
\begin{align*}
q_n(p) &\leq q(p) +  \eps\sqrt{q(p)},\\
q_n(p) &\geq q(p) -  \eps\sqrt{q(p)}.
\end{align*}
whereas, \cite{CortesGM13} analyzes the event where it is bounded it by:
\begin{align*}
q_n(p) &\leq  \tilde O\bigl(q(p) + \sqrt{q(p)/n} + \nicefrac{1}{n}\bigr),\\
q_n(p) &\geq \tilde\Omega\bigl(q(p) - \sqrt{q_n(p)/n} - \nicefrac{1}{n}\bigr)
\end{align*}
Thus, the main difference is the additive $\nicefrac{1}{n}$ term in
the bound from \cite{CortesGM13}. In the context of uniform convergence of revenues, it is crucial
to use the upper bound on the empirical quantile as we do, 
as it guarantees that large prices will not
overfit, which is the main challenge in proving uniform convergence
in this context. In particular, the upper bound from \cite{CortesGM13}
does not provide any guarantee on the revenues of prices
$ p >\!> n$, as for such prices $p\cdot \nicefrac{1}{n} >\!> 1$.

It is also worth pointing out
that our lower bound on the empirical quantile
implies that with high probability the quantile of the maximum sampled point
is at least $\nicefrac{1}{n^2}$ (or more generally, at least $\nicefrac{1}{n^{1/\alpha}}$ when $\alpha\neq\nicefrac{1}{2}$),
while the bound from \cite{CortesGM13} does not imply any non-trivial lower bound.

Another, more qualitative difference is that unlike the bounds in \cite{CortesGM13} that
apply for general VC classes, our bound is tailored 
for the class of thresholds (corresponding to CDF/quantiles),
and does not extend even to other classes of VC dimension 1
(see the discussion after Theorem~\ref{t:qual}).

\paragraph{Uniform convergence of revenues.}
The problem of \emph{revenue maximization} is a central problem in
economics and Algorithmic Game Theory (AGT). The seminal work of \cite{Myerson}
shows that given a valuation distribution for a single good, the revenue-maximizing selling mechanism for this good
is a posted-price mechanism. In the recent years, there has been a
growing interest in the case where the valuation distribution is
unknown, but the seller observes samples drawn from it. Most papers in this
direction assume that the distribution meets some tail condition
that is considered ``natural'' within the algorithmic game theory
community, such as boundedness \citep{Morgenstern-Roughgarden,Roughgarden-Schrijvers,Morgenstern-Roughgarden-b,Balcan-Sandholm-Vitercik,Gonczarowski-Nisan,Devanur-Huang-Psomas}\footnote{The analysis of \cite{Balcan-Sandholm-Vitercik} assumes a bound on the realized revenue (from any possible valuation profile) of any mechanism/auction in the class that they consider. For the class of posted-price mechanisms, this is equivalent to assuming a bound on the support of the valuation distribution. Indeed, for any valuation $v$, pricing at $v$ gives realized revenue $v$ (from the valuation $v$), and so unbounded valuations (together with the ability to post unbounded prices) imply unbounded realized revenues.},
such as a condition known as Myerson-regularity
\citep{DhangwatnotaiRY15,Huang-Mansour-Roughgarden,Cole-Roughgarden,Devanur-Huang-Psomas},
or such as a condition known as monotone hazard rate
\citep{Huang-Mansour-Roughgarden}.\footnote{Both Myerson-regularity
and monotone hazard rate are conditions on the second derivative of
the revenue as a function of the quantile of the underlying
distribution. In particular, they impose restrictions on the tail of
the distribution.} These papers then go on to derive computation- or
sample-complexity bounds on learning an optimal price (or an optimal selling mechanism from a given class) for a
distribution that meets the assumed condition.

A recurring theme in statistical learning theory
is that learnability guarantees are derived via a,
sometimes implicit, uniform convergence bound.
However, this has not been the case in the context of revenue learning.
Indeed, while some papers that studied bounded distributions~\citep{Morgenstern-Roughgarden,Roughgarden-Schrijvers,Morgenstern-Roughgarden-b,Balcan-Sandholm-Vitercik}
did use uniform convergence bounds as part of their analysis,
other papers, in particular those that considered unbounded distributions, had to bypass
the usage of uniform convergence by more specialized arguments.
This is due to the fact that
many unbounded distributions
do not satisfy any uniform convergence bound.
As a concrete example, the (unbounded, Myerson-regular) \emph{equal revenue distribution}\footnote{This is a distribution that satisfies the special property that all prices have the same expected revenue.} has an infinite expectation and
therefore, by our Theorem~\ref{t:UC1moment}, satisfies no uniform convergence, even
in the limit. Thus,
it turns out that
the works that studied the popular class of Myerson-regular distributions \citep{DhangwatnotaiRY15,Huang-Mansour-Roughgarden,Cole-Roughgarden,Devanur-Huang-Psomas} indeed could not have hoped to
establish learnability via a uniform convergence argument.
For instance, the way \cite{DhangwatnotaiRY15} and \cite{Cole-Roughgarden} establish learnability for Myerson-regular distributions
is by considering the guarded ERM algorithm (an algorithm that chooses an empirical revenue maximizing price that is smaller than, say, the $\sqrt{n}$th largest sampled price), and proving a uniform convergence bound,
not for all prices, but only for prices that are, say, smaller than
the $\sqrt{n}$th largest sampled price, and then arguing that larger
prices are likely to have a small empirical revenue, compared to the
guarded empirical revenue maximizer. This means that the guarded ERM
will output a good price, but it does not (and cannot) imply uniform convergence
for all prices.

We complement the extensive literature surveyed above in a few ways.
The first is
generalizing the revenue maximization problem to a revenue
estimation problem, where the goal is to uniformly estimate the revenue
of all possible prices, when no bound on the possible valuations is given (or even exists). The problem of revenue estimation arises naturally when the seller has additional considerations when pricing her good, such
as regulations that limit the price choice, bad publicity if the price is too high (or, conversely, damage to prestige if the price is too low),
or willingness to suffer some revenue loss for better market penetration (which may translate to more revenue in the future).
In such a case, the seller may wish to estimate the revenue loss due to posting a discounted (or inflated) price.

The second, and most important, contribution to the above literature is that we
consider arbitrary distributions rather than very specific and limited
classes of distributions (e.g., bounded, Myerson-regular, monotone hazard rate, etc.). Third, we derive finite sample bounds in
the case that the expected valuation is bounded for some moment
larger than~1.
We further derive a zero-one law for uniform convergence in the limit
that depends on the finiteness of the first moment.
Technically, our bounds are based on an additive error rather than multiplicative
ones, which are popular in the AGT community.

\subsection{Paper organization}
The rest of the paper is organized as follows.
We begin by presenting the application of our Submultiplicative Glivenko-Cantelli
to revenue estimation in Section~\ref{sec:rev}. In Section~\ref{sec:GC},
we prove the Submultiplicative Glivenko-Cantelli variant, and discuss some
extensions of it.
Section~\ref{sec:discus} contains a discussion 
and possible directions of future work.
Some of the proofs are deferred to the appendices.

\section{Uniform Convergence of Empirical Revenues}
\label{sec:rev}

In this section we demonstrate an application
of our Submultiplicative Glivenko-Cantelli variant by establishing
uniform convergence bounds for a family of unbounded random variables
in the context of revenue estimation.

\subsection{Model}
Consider a good $g$ that we wish to post a price for.
Let $V$ be a random variable that models the valuation of a random
consumer for $g$. Technically, it is assumed that $V$ is
a nonnegative random variable, and we denote by $\mu$ its induced distribution over $\mathbb{R}^+$.
A consumer who values $g$ at a valuation $v$ is willing to buy the good at any
price $p\leq v$, but not at any higher price. This implies that the realized revenue to the seller from a (posted) price $p$
is the random variable $p \cdot 1_{\{p\leq V\}}$.
The \emph{quantile} of a value $v \in \R^+$ is
\[q(v) = q(v;\mu) \eqdef\mu\bigl(\{x: x \geq v\}\bigr).\]
This models the fraction of the consumers in the population that are willing
to purchase the good if priced at $v$.
The expected revenue from a (posted) price $p \in \R^+$ is
\[r(p) = r(p;\mu) \eqdef \Ex_{\mu}\bigl[p \cdot 1_{\{p\leq V\}}\bigr] =  p \cdot q(p).\]

Let $V_1,V_2,\ldots$ be a sequence of i.i.d.\ valuations drawn from $\mu$,
and let $v_1,v_2,\ldots$ be their realizations.
The \emph{empirical quantile} of a value $v \in \R^+$ is
\[q_n(v) =  q(v;\mu_n) \eqdef
 \tfrac{1}{n}\cdot\bigl\lvert\{1 \leq i \leq n : v_i \geq v \} \bigr\rvert.\]
The \emph{empirical revenue} from a price $p \in \R^+$ is
\[r_n(p) = r(p;\mu_n)\eqdef\Ex_{\mu_n}\bigl[p \cdot 1_{\{p\leq V\}}\bigr] =p \cdot q_n(p).\]
The revenue estimation error for a given sample of size $n$ is
\[
\epsilon_n\eqdef \sup_p \bigl|r_n(p)-r(p)\bigr|.
\]

It is worth highlighting the difference between revenue estimation
and revenue maximization.
Let~$p^*$ be a price that maximizes the revenue,
i.e., $p^*\in\arg\sup_p r(p)$. The maximum revenue is $r^*=r(p^*)$.
The goal in many works in revenue maximization is to find a price
$\hat{p}$ such that $r^*-r(\hat{p})\leq \epsilon$, or alternatively, to
bound $\nicefrac{r^*}{r(\hat{p})}$.

Given a revenue-estimation error $\epsilon_n$, one can clearly
maximize the revenue within an additive error of
$2\epsilon_n$ by simply posting a price
$p^*_n\in\arg\max_p r_n(p)$, thereby attaining revenue $r^*_n=r(p^*_n)$. This follows since
\[
r^*_n = r(p^*_n) \ge r_n(p^*_n)-\epsilon_n \ge r_n(p^*)-\epsilon_n \ge r(p^*)-2\epsilon_n=r^*-2\epsilon_n.
\]
Therefore, good revenue estimation implies good revenue maximization.

We note that the converse does not hold. Namely, there are
distributions for which revenue maximization is trivial but
revenue estimation is impossible. One such case is the \emph{equal revenue
distribution}, where all values in the support of $\mu$ have the same
expected revenue. For such distributions, the problem of revenue maximization becomes trivial,
since any posted price is optimal.
However, as follows from Theorem~\ref{t:UC1moment},
since the expected revenue of such distributions is infinite,
almost never do the empirical revenues uniformly converge to the true revenues.

\subsection{Quantitative bounds on the uniform convergence rate}\label{sec:rev-qual}

Recall that we are interested in deriving sample bounds that would
guarantee uniform convergence for the revenue estimation problem.
We will show that given an upper bound on the $k$th moment of $V$ for some $k>1$, we can derive a finite
sample bound.
To this end we utilize our Submultiplicative Glivenko-Cantelli Bound
(Theorem~\ref{thm:samplecomplexity}).

We also consider the case of $k=1$, namely that $\Ex[V]$ is bounded,
and show that in this case there is still uniform convergence in the
limit, but that there cannot be any guarantees on the convergence
rate. Interestingly, it turns out that $\Ex[V] < \infty$ is
not only sufficient but also necessary so that in the limit, the empirical revenues uniformly
converge to the true revenues (see Section~\ref{sec:UC1MOM}).

We begin by showing that bounds on the $k$th moment for $k>1$ yield
explicit bounds on the convergence rate. It is convenient to
parametrize by setting $k=1+\theta$, where $\theta > 0$.
\begin{theorem}\label{thm:rev}
Let $\Ex_{V\sim \mu}[V^{1+\theta}]\leq C$ for some $\theta>0$ and
$C\geq 1$, and let $\eps,\delta\in (0,1)$. Set\footnote{The $\tilde
O$ conceals low order terms.}
\begin{equation}\label{eq:quant1}
n_0=\tilde O\Biggl(\frac{\ln(\nicefrac{1}{\delta})}{\eps^2}
C^{\frac{2}{1+\theta}}\biggl(\frac{6\cdot C^{\frac{1}{1+\theta}}}{\eps\delta\ln\bigl(1+\nicefrac{\theta}{2}\bigr)}\biggr)^{\nicefrac{4}{\theta}}\Biggr).
\end{equation}
For any $n\geq n_0$, with probability at least $1-\delta$,
\[
\forall v:\;\;\; \bigl|r_n(v)-r(v)\bigr|\leq \eps.
\]
\end{theorem}

Note that when $\theta$ is large, this  bound approaches the
standard $O\Bigl(\frac{\ln(\nicefrac{1}{\delta})}{\eps^2}\Bigr)$ sample
complexity bound of the additive Glivenko-Cantelli. For example, if
all moments are uniformly bounded, then the convergence is roughly
as fast as in standard uniform convergence settings (e.g., VC-dimension based bounds).

The proof of Theorem~\ref{thm:rev} follows from
Theorem~\ref{thm:samplecomplexity} and the next proposition, which
reduces bounds on the uniform convergence rate of the empirical revenues to our
Submultiplicative Glivenko-Cantelli.
\begin{prop}\label{prop:reduc}
Let $\Ex_{V\sim \mu}[V^{1+\theta}]\leq C$ for some $\theta>0$ and
$C\geq 1$, and let $\eps,\delta\in (0,1)$. Then,
\[\Pr \Bigl[\exists v :~\bigl\lvert r(v) - r_n(v)\bigr\rvert > \eps\Bigr] \leq
\Pr \biggl[\exists v :~\bigl\lvert q(v) - q_n(v)\bigr\rvert >
\frac{\eps}{C^{\frac{1}{1+\theta}}}q(v)^{\frac{1}{1+\theta}}\biggr].\]
\end{prop}
Thus, to prove Theorem~\ref{thm:rev}, we first note that Theorem~\ref{thm:samplecomplexity} (as well as Theorem~\ref{t:qual}) also holds when $F_n$ and $F$ are respectively replaced in the definition of $n_0$ with $q_n$ and $q$ (indeed, applying Theorem~\ref{thm:samplecomplexity} to the measure $\mu'$ defined by $\mu'(A)\eqdef\mu\bigl(\{-a\mid a\in A\}\bigr)$ yields the required result with regard to the measure $\mu$). We then plug 
$\eps\leftarrow\frac{\eps}{C^{\frac{1}{1+\theta}}}$ and $\alpha\leftarrow\frac{1}{1+\theta}$ into this variant of Theorem~\ref{thm:samplecomplexity} to yield a bound on the right-hand side of the inequality in Proposition~\ref{prop:reduc}, whose application concludes the proof.
\begin{proof}[Proof of Proposition~\ref{prop:reduc}]
By Markov's inequality:
\begin{equation}\label{eq:secmom}
q(v) = \Pr[V\geq v]  =\Pr[V^{1+\theta}\geq v^{1+\theta}]  \leq
\frac{C}{v^{1+\theta}}.
\end{equation}
Now,
\begin{align*}
\Pr \Bigl[\exists v :~\bigl\lvert r(v) - r_n(v)\bigr\rvert > \eps\Bigr]
&=
\Pr \Bigl[\exists v :~\bigl\lvert v\cdot q(v) - v\cdot q_n(v)\bigr\rvert > \eps\Bigr]\\
&=
\Pr \Bigl[\exists v :~\bigl\lvert v\cdot q(v) - v\cdot q_n(v)\bigr\rvert >
 \frac{\eps}{(v^{1+\theta}\cdot q(v))^{\frac{1}{1+\theta}}}(v^{1+\theta}\cdot q(v))^{\frac{1}{1+\theta}}\Bigr]\\
&\leq
\Pr \Bigl[\exists v :~\bigl\lvert v\cdot q(v) - v\cdot q_n(v)\bigr\rvert >
 \frac{\eps}{C^{\frac{1}{1+\theta}}}(v^{1+\theta}\cdot q(v))^{\frac{1}{1+\theta}}\Bigr]\\
&= \Pr \Bigl[\exists v :~\bigl\lvert q(v) - q_n(v)\bigr\rvert >
\frac{\eps}{C^{\frac{1}{1+\theta}}}q(v)^{\frac{1}{1+\theta}}\Bigr].
\end{align*}
where the inequality follows from  Equation~(\ref{eq:secmom}).
\end{proof}

\subsection{A qualitative characterization of uniform convergence}
\label{sec:UC1MOM} The sample complexity bounds in
Theorem~\ref{thm:rev} are meaningful as long as $\theta > 0$, but
deteriorate drastically as $\theta\to 0$. Indeed, as the following
example shows, there is no bound on the uniform convergence sample
complexity that depends only on the first moment of $V$, i.e., its
expectation.

Consider a distribution $\eta_p$ so that with probability $p$ we
have $V=\nicefrac{1}{p}$ and otherwise $V=0$. Clearly, $\Ex[V]=1$. However, we
need to sample $m_p=O(\nicefrac{1}{p})$ valuations to see a single nonzero value. Therefore,
there is no bound on the sample size $m_p$ as a function of the
expectation, which is simply~$1$.

We can now consider the higher moments of $\eta_p$. Consider the $k$th moment, for
$k=1+\theta$ and $\theta>0$, so $k>1$. For this moment, we have
$A_{p,\theta}=\Ex[V^{1+\theta}]=p^{\theta/(1+\theta)}$, which
implies that $m_p=O\bigl(1/(A_{p,\theta})^{(1+\theta)/\theta}\bigr)$. This
does allow us to bound $m_p$ as a function of $\theta$ and
$\Ex[V^{1+\theta}]$, but for small $\theta$ we have a huge exponent
of approximately $\nicefrac{1}{\theta}$.

While the above examples show that there cannot be a bound
on the sample size as a function of the expectation of the value,
it turns out that there is a very tight connection between the first moment
and uniform convergence:

\begin{theorem}\label{t:UC1moment}
The following dichotomy holds for a distribution $\mu$ on $\mathbb{R}^+$:
\begin{enumerate}
\item\label{item1} If $\Ex_\mu[V] < \infty$, then almost surely $\lim_{n\to\infty}\sup_{v}\bigl\lvert r(v) - r_n(v)\bigr\rvert=0$.
\item\label{item2} If $\Ex_\mu[V] = \infty$, then almost never $\lim_{n\to\infty}\sup_{v}\bigl\lvert r(v) - r_n(v)\bigr\rvert =0$.
\end{enumerate}
That is, the empirical revenues uniformly converge to the true revenues if and only if
$\Ex_{\mu}[V]<\infty$.
\end{theorem}
We use the following basic fact in the Proof of Theorem~\ref{t:UC1moment}:
\begin{lemma}\label{lem:basic}
Let $X$ be a nonnegative random variable. Then
\[    \sum_{n=1}^{\infty}\Pr[X\geq n]\leq \Ex[X] \leq  \sum_{n=0}^{\infty}\Pr[X\geq n].\]
\end{lemma}
\begin{proof}
Note that:
\[\sum_{n=1}^\infty 1_{\{X\geq n\}}= \lfloor X \rfloor \leq  X \leq \lfloor X\rfloor + 1 = \sum_{n=0}^\infty 1_{\{X\geq n\}}.\]
The lemma follows by taking expectations.
\end{proof}

\begin{proof}[Proof of Theorem~\ref{t:UC1moment}]
We start by proving item \ref{item2}. Let $\mu$ be a distribution
such that $\Ex_{\mu}\bigl[V\bigr] = \infty$. If $\sup_v v\cdot
q(v)=\infty$ then for every realization $v_1,\ldots, v_n$ there is
some $v\geq \max\{v_1,\ldots,v_n\}$ such that $v\cdot q(v) \geq 1$,
but $v\cdot q_n(v)=0$. So, we may assume $\sup_v v\cdot q(v) <
\infty$. Without loss of generality we may assume that $\sup_v
v\cdot q(v)=\nicefrac{1}{2}$ by rescaling the distribution if
needed. Consider the sequence of events $E_1,E_2,\ldots$ where
$E_n$ denotes the event that $V_n \geq n$. Since
$\Ex_{\mu}\bigl[V\bigr] = \infty$, Lemma~\ref{lem:basic} implies
that
$\sum_{n=1}^\infty \Pr[E_n]=\infty$.
Thus, since these events are independent,
the second Borel-Cantelli Lemma~\citep{Borel09,Cantelli17} implies that almost surely,
infinitely many of them occur and so infinitely often
\[V_n\cdot q_n(V_n) \geq 1 \geq V_n\cdot q(V_n) + \tfrac{1}{2}.\]
Therefore,  the probability that $v \cdot q_n(v)$ uniformly converge
to $v \cdot q(v)$ is $0$.

Item \ref{item1} follows from the following monotone domination
theorem:
\begin{theorem}\label{thm:envelope}
Let $\mathcal{F}$
be a family of nonnegative monotone functions,
and let $F$ be an upper envelope\footnote{$F$ is an upper envelope for $\mathcal{F}$
if $F(v)\geq f(v)$ for every $v\in V$ and $f\in\mathcal{F}$.} for $\mathcal{F}$.
If $\Ex_\mu[F] < \infty$, then almost surely:
\[\lim_{n\to\infty}\sup_{f\in\mathcal{F}}\bigl\lvert \Ex_\mu[f]-\Ex_{\mu_n}[f]\bigr\rvert=0.\]
\end{theorem}
Indeed, item \ref{item1} follows by plugging $\mathcal{F} =
\bigl\{v\cdot1_{x\geq v} : v\in\mathbb{R}^+\bigr\}$, which is
uniformly bounded by the identity function $F(x)=x$. Now, by
assumption $\Ex_\mu[F] < \infty$, and therefore, almost surely
\[\lim_{n\to\infty}\sup_{v\in\mathbb{R}^+}\bigl\lvert r(v)-r_n(v)\bigr\rvert=
\lim_{n\to\infty}\sup_{f\in\mathcal{F}}\bigl\lvert \Ex_\mu[f]-\Ex_{\mu_n}[f]\bigr\rvert=0.\qedhere\]
\end{proof}

Theorem~\ref{thm:envelope} follows by known results in the theory of empirical processes
(for example, with some work it can be proved using Theorem 2.4.3 from \cite{VaartWellner96}).
For completeness, we give a short and basic proof in Appendix~\ref{app:envelope}.

\section{Submultiplicative Glivenko-Cantelli}
\label{sec:GC}
In this section $\mu$ is a fixed but otherwise arbitrary distribution, with CDF $F$
and empirical CDF $F_n$.

Theorem \ref{t:qual} is a corollary of the following lemma, which
gives a quantitative bound on the confidence parameter~$\delta$.

\begin{lemma}\label{lem:quant}
Let $n \in \N$, $\eps > 0$ and $\alpha, p,q \in (0,1)$.
Assume that $n\geq \eps^{-\frac{1}{1-\alpha}}$ and $ p \leq \min\{\eps^{\frac{1}{1-\alpha}}, \nicefrac{1}{e}\}$. Then,
\[\Pr \Bigl[\exists t :~\bigl\lvert F(t) - F_n(t)\bigr\rvert > \eps\cdot F(t)^\alpha \Bigr]\leq q +
\Biggl\lceil\frac{\ln\ln(\frac{n}{q})}{\ln(\frac{1+\alpha}{2\alpha})}\Biggr\rceil \frac{p^{ \frac{1-\alpha}{2}}}{{\eps}} +
2\exp\bigl(-2n (\eps p^\alpha)^2\bigr). \]
\end{lemma}

Note that $p$ and $q$ appear only on the right-hand side, and therefore can be ``tuned'' in order to minimize the upper bound.
Our proof of Lemma~\ref{lem:quant} uses Theorem~\ref{t:massart}.
To better understand the parameters,
we state the following corollary (whose first item is stronger than Theorem~\ref{t:qual}).

\begin{corollary}\label{c:weakGC}
{There are constants $c_1,c_2 > 0$ so that the following holds.
\begin{enumerate}
\item If $\alpha(n)\leq 1-c_1 \cdot \frac{{\ln\ln(n)}}{\ln(n)}$,
then the probability of the event
\[ \forall t  : \bigl\lvert F(t) - F_n(t)\bigr\rvert \leq \eps\cdot F(t)^{\alpha(n)} \]
tends to $1$ as $n$ tends to $\infty$.
\item If $\alpha(n)\geq 1-c_2 \cdot \frac{{1}}{\ln(n)}$ and $\mu$ is uniform over $[0,1]$,
then the probability of the event
\[ \forall t  : \bigl\lvert F(t) - F_n(t)\bigr\rvert \leq \frac{1}{10}\cdot F(t)^{\alpha(n)} \]
is at most $\nicefrac{1}{2}$, for all $n\geq 2$.
\end{enumerate}}
\end{corollary}
{We leave as an open question to determine the behavior of these probabilities when 
\[\alpha(n)\in\left[1-c_1 \cdot \frac{{\ln\ln(n)}}{\ln(n)},1-c_2 \cdot \frac{{1}}{\ln(n)}\right].\]}
Corollary~\ref{c:weakGC} is proven in Appendix~\ref{app:weakGC}.
\begin{proof}[Proof of Lemma~\ref{lem:quant}]
Let $\eps,\alpha,q,p,n$ be as in the statement of the lemma.
We partition the event in question to three parts,
depending on the value of $t$ as follows.
Partition $\R$ to
\[I_{[0,q/n]}=\left\{t\in\mathbb{R}:~ 0\leq F(t)\leq\frac{q}{n}\right\}, \ I_{(q/n,p)}=\left\{t\in\mathbb{R}:~ \frac{q}{n}< F(t)\leq p\right\}\]
and
\[I_{[p,1]}=\left\{t\in\mathbb{R}:~ p< F(t)\leq 1\right\}.\]
There are three corresponding events $E_{[0,q/n]},E_{(q/n,p)}$
and $E_{[p,1]}$; for example,
$E_{[0,q/n]}$ is the event that $\exists t \in I_{[0,q/n]} : B(t) = 1$,
where $B(t)$ is the indicator of
$\bigl\lvert F(t) - F_n(t)\bigr\rvert > \eps\cdot F(t)^\alpha$.

The following three claims bound from above
the probabilities of these three events.
The three claims and the union bound complete the proof
of the theorem.

\begin{claim}
\label{claim1}
 $\Pr\bigl[E_{[0,q/n]}\bigr] \leq q$.
\end{claim}

\begin{proof}
Let $t \in I_{[0,\nicefrac{q}{n}]}$ be so that $B(t)=1$.
For any $t\in I_{[0,\nicefrac{q}{n}]}$ we have that $F(t)\leq
\nicefrac{q}{n}\leq \nicefrac{1}{n} \leq \eps^{\frac{1}{1-\alpha}}$,
where the last inequality is by our assumption on $n$ and
$\epsilon$. This implies that $F(t) \leq \eps F(t)^\alpha$.
Since $B(t)=1$ it must be the case that $F_n(t)
> F(t) + \eps\bigl(F(t)\bigr)^\alpha \geq 0$,
and therefore at least one sample $x_i$ satisfies $x_i\leq t \leq
\nicefrac{q}{n}$.
Now, by the union bound,
\[\Pr\bigl[E_{[0,q/n]}\bigr] \leq
\Pr\bigl[\exists i \in [n] : ~ x_i \in I_{[0,q/n]}\bigr]
 \leq n\cdot \frac{q}{n}=q.\qedhere\]
\end{proof}

\begin{claim}\label{claim:main}
$\Pr\bigl[E_{(q/n,p)}\bigr] \leq \left\lceil\frac{\ln\ln(\frac{n}{q})}{\ln(\frac{1+\alpha}{2\alpha})}\right\rceil \frac{p^{ \frac{1-\alpha}{2} }}{\eps}$.
\end{claim}

\begin{proof}
If $\nicefrac{q}{n} \geq p$ then this event is empty and its probability is $0$.
Therefore, assume that $\nicefrac{q}{n} < p$, and that this event is not empty.

For all $t\in I_{(q/n,p)}$ since $F(t)\leq p \leq
\eps^{\frac{1}{1-\alpha}}$ we have $F(t) - \eps F(t)^\alpha \leq 0$.
So, it suffices to consider the event
\[\exists t\in I_{(q/n,p)}: F_n(t) > F(t) + \eps\cdot F(t)^\alpha.\]

Consider the decreasing sequence of numbers
$p_0,p_1,\ldots,p_m$ defined by
\[p_i = p^{\bigl(\frac{1+\alpha}{2\alpha}\bigr)^{i}},\]
where $m$ is such that $p_m < \nicefrac{q}{n} \leq p_{m-1}$.
Since $p \leq 1/e$,
we can bound $m \leq \left\lceil\frac{\ln\ln(\frac{n}{q})}{\ln(\frac{1+\alpha}{2\alpha})}\right\rceil$.
Let
\[t_i= \inf \bigl\{t \in I_{(q/n,p)} : F(t) \geq p_i\bigr\}.\]
Let $F_n^-(t) = \mu_n\bigl(\{x:x < t\}\bigr)$. We claim that
\begin{equation*}
\exists t\in I_{(q/n,p)}: B(t) =1 \implies \exists i < m : F_n^-(t_i) > \eps\cdot p_{i+1}^\alpha,
\end{equation*}
Indeed, assume that $t\in I_{(q/n,p)}$ satisfies $F_n(t) > F(t) +
\eps\cdot F(t)^\alpha$. Since $p_m< t < p_0$, there is some $0\leq
i\leq m-1$ such that $p_{i+1}\leq F(t)< p_{i}$. Note that $t_{i+1}
\leq t < t_i$. Indeed, $t_{i+1} \leq t$ follows since $p_{i+1} \leq
F(t)$, and $t < t_i$ follows since $F(t_i) \geq p_i$ (which is
implied by right continuity of $F$). Hence,
\begin{align*}
F_n^-(t_i) \geq
F_n(t)  >
F(t) + \eps\cdot F(t)^\alpha
 \geq p_{i+1} + \eps\cdot p_{i+1}^\alpha \geq \eps\cdot p_{i+1}^\alpha .
\end{align*}

It hence remains to upper bound the union of these events.
Note that
\[\Ex\bigl[F_n^-(t_i)\bigr]=\mu\bigl(\{x:x<t_i\}\bigr)\leq p_i.\]
Therefore, by Markov's inequality:
\[
\Pr\Bigl[F_n^-(t_i) > \eps\cdot p_{i+1}^\alpha\Bigr] \leq
\frac{p_i}{\eps p_{i+1}^\alpha}
=\frac{1}{\eps}
p^{\bigl(\frac{1-\alpha}{2}\bigr)\bigl(\frac{1+\alpha}{2\alpha}\bigr)^{i}}.
\]
By the union bound,
\begin{align*}
\Pr \bigl[ \exists i < m : F_n(t_i) > p_{i+1} + \eps\cdot p_{i+1}^\alpha \bigr]&\leq
\frac{1}{\eps}\sum_{i=0}^{m-1}p^{\bigl(\frac{1-\alpha}{2}\bigr)\bigl(\frac{1+\alpha}{2\alpha}\bigr)^{i}}
&\leq
\frac{m}{{\eps}}p^{\frac{1-\alpha}{2}}
&\leq
\Biggl\lceil\frac{\ln\ln(\frac{n}{q})}{\ln(\frac{1+\alpha}{2\alpha})}\Biggr\rceil \frac{p^{\frac{1-\alpha}{2}}}{{\eps}}.\qedhere
\end{align*}
\end{proof}

\begin{claim}
\label{claim3} $\Pr \bigl[E_{[p,1]}\bigr] \leq 2\exp\bigl(-2n (\eps
p^\alpha)^2\bigr)$.
\end{claim}

\begin{proof}
For all $t\in I_{[p,1]}$ we have $F(t)^\alpha\geq p^\alpha$.
The claim follows by Theorem~\ref{t:massart}.
\end{proof}
\noindent
Lemma~\ref{lem:quant} follows from combining Claims \ref{claim1},
\ref{claim:main}, and \ref{claim3}.
\end{proof}

\section{Discussion}
\label{sec:discus}

Our main result is a submultiplicative variant of the
Glivenko-Cantelli Theorem, which allows for tighter convergence
bounds for extreme values of the CDF.
We show that for the revenue learning setting our submultiplicative
bound can be used to derive uniform convergence sample complexity bounds, 
assuming a finite bound on the $k$th moment of the valuations, for any (possibly fractional) $k>1$.
For uniform convergence in the limit, we give a complete
characterization, where uniform convergence almost surely occurs if
and only if the first moment is finite.

It would be interesting to find other applications of our
submultiplicative bound in other settings.
A potentially interesting  direction is to consider
unbounded loss functions (e.g.,\ the squared-loss, or log-loss).
Many works circumvent the unboundedness in such cases
by ensuring (implicitly) that the losses are bounded, e.g., through
restricting the inputs and the hypotheses.
Our bound offers a different perspective of
addressing this issue. 
In this paper we consider revenue learning,
and replace the boundedness assumption 
by assuming bounds on
higher moments.
An interesting challenge is to prove uniform convergence bounds for
other practically interesting settings. 
One such setting  might be estimating the effect of outliers
(which correspond to the extreme values of the loss).

In the context of revenue estimation, this work only considers 
the most na\"{i}ve estimator, namely of estimating the revenues
by the empirical revenues. One can envision other estimators, for
example ones which regularize the extreme tail of the sample. Such
estimators may have a potential of better guarantees or better
convergence bounds. In the context of uniform convergence of selling mechanism revenues, this work only considers the basic class of posted-price mechanisms. While
for one good and one valuation distribution, it is always possible to maximize revenue via a selling mechanism of this class, this is not the case in more complex auction environments. While in many more-complex environments, the revenue-maximizing mechanism/auction is still not understood well enough, for environments where it is understood \citep{Cole-Roughgarden,Devanur-Huang-Psomas,Gonczarowski-Nisan} (as well as for simple auction classes that do not necessarily contain a revenue-maximizing auction \citep{Morgenstern-Roughgarden-b,Balcan-Sandholm-Vitercik}) it would also be interesting to study relaxations of the restrictive tail or boundedness assumptions currently common in the literature.

\section*{Acknowledgments}

The research of Noga Alon is supported in part by an ISF grant and
by a GIF grant. 
Yannai Gonczarowski is supported by the Adams
Fellowship Program of the Israel Academy of Sciences and Humanities;
his work is supported by ISF grant 1435/14 administered by the
Israeli Academy of Sciences and by Israel-USA Bi-national Science
Foundation (BSF) grant number 2014389; this project has received
funding from the European Research Council (ERC) under the European
Union's Horizon 2020 research and innovation programme (grant
agreement No 740282). 
The research of Yishay Mansour was supported
in part by The Israeli Centers of Research Excellence (I-CORE)
program (Center  No.\ 4/11), by a grant from the Israel Science
Foundation, and by a grant from United States-Israel Binational
Science Foundation (BSF).
The research of Shay Moran is supported by the National Science Foundations
and the Simons Foundations.
The research of Amir Yehudayoff is supported by ISF grant 1162/15.

\bibliographystyle{abbrvnat}
\bibliography{paper}

\clearpage

\appendix

\section{Proof of Theorem~\ref{thm:samplecomplexity}}
\label{app:thm:rev}

\begin{proof}
Let $\eps,\delta\leq \nicefrac{1}{4}$ and $\alpha < 1$.
By Lemma~\ref{lem:quant},
\begin{equation}\label{eq:three-summands}
\Pr \Bigl[\exists t :~\bigl\lvert F(t) - F_n(t)\bigr\rvert > \eps\cdot F(t)^\alpha \Bigr]\leq q + \Biggl\lceil\frac{\ln\ln(\frac{n}{q})}{\ln(\frac{1+\alpha}{2\alpha})}\Biggr\rceil \frac{p^{ \frac{1-\alpha}{2}}}{{\eps}} + 2\exp\bigl(-2n (\eps p^\alpha)^2\bigr)
\end{equation}
for every $ q \leq 1$, $n\geq \eps^{-\frac{1}{1-\alpha}}$ and $ p \leq \eps^{\frac{1}{1-\alpha}}$.

Set $q,p$ so that each of the first two summands in Equation~\ref{eq:three-summands} is at most
$\delta$. Specifically, $q=\delta$, and
\begin{enumerate}
\item if $\frac{\ln\ln(\frac{n}{\delta})}{\ln(\frac{1+\alpha}{2\alpha})}\geq 1$ then set
$p=\Bigl({\eps\delta}\cdot
 \frac{\ln\bigl(\frac{1+\alpha}{2\alpha}\bigr)}{2\ln\ln(\frac{n}{\delta})}\Bigr)^{\frac{2}{1-\alpha}}$, and
\item if $\frac{\ln\ln(\frac{n}{\delta})}{\ln(\frac{1+\alpha}{2\alpha})} < 1$
then set
$p=(\eps\delta)^{\frac{2}{1-\alpha}}$.
\end{enumerate}

Note that indeed the requirements
$n\geq \eps^{-\frac{1}{1-\alpha}}$ and $ p \leq\eps^{\frac{1}{1-\alpha}}$
are satisfied by $p$ and by the desired $n$ (from the theorem
statement).

Plugging these $p$ and $q$ in the third summand in Equation~\ref{eq:three-summands} yields:
\[2\exp\left(-2n\eps^2
 \left({\eps\delta}\cdot
 \frac{\ln\bigl(\frac{1+\alpha}{2\alpha}\bigr)}{2\ln\ln(\frac{n}{\delta})}\right)^{\frac{4\alpha}{1-\alpha}}\right)\]
when $\frac{\ln\ln(\frac{n}{\delta})}{\ln(\frac{1+\alpha}{2\alpha})}\geq 1$, or
\[2\exp\biggl(-2n\eps^2
 ({\eps\delta})^{\frac{4\alpha}{1-\alpha}}\biggr)\]
otherwise. In order for the above to be at most $\delta$, it suffices
that
\[2n\eps^2
\left({\eps\delta}\cdot
\frac{\ln\bigl(\frac{1+\alpha}{2\alpha}\bigr)}
{2\ln\ln(\frac{n}{\delta})}\right)^{\frac{4\alpha}{1-\alpha}}
\geq \ln(\nicefrac{2}{\delta})\]
when $\frac{\ln\ln(\frac{n}{\delta})}{\ln(\frac{1+\alpha}{2\alpha})}\geq 1$, or
\[2n\eps^2 ({\eps\delta})^{\frac{4\alpha}{1-\alpha}} \geq \ln(\nicefrac{2}{\delta})\]
otherwise.
The second case implies an explicit bound of
\begin{equation}\label{eq:case2}
n \geq \frac{ \ln(\nicefrac{2}{\delta})}{2\eps^2}({\eps\delta})^{-\frac{4\alpha}{1-\alpha}}.
\end{equation}
To get an explicit bound on $n$ in the first case, we need to solve
a recursion of the following type:
find a lower bound on $n$ so that the following inequality holds:
\[n\geq D \bigl(\ln\ln(E\cdot n)\bigr)^F,\]
where $D\geq 0$, $E\geq 4$, $F\geq 0$. (Here
$D=\frac{\ln(\nicefrac{2}{\delta})}{2\eps^2}
\Bigl(\frac{\eps\delta}{2}\cdot
{\ln\bigl(\frac{1+\alpha}{2\alpha}\bigr)}\Bigr)^{-\frac{4\alpha}{1-\alpha}}$,
$E=\frac{1}{\delta}$, $F=\frac{4\alpha}{1-\alpha}$.) Setting
\begin{equation}\label{eq:case1}
n\geq (D+1)\Bigl(10\bigl(\ln(D+4) + \ln(F+4) + \ln(E)\bigr)\Bigr)^F
= (D+1)\Biggl(10\cdot\ln\left(4\cdot\frac{D+4}{\delta(1-\alpha)}\right)\Biggr)^{\frac{4\alpha}{1-\alpha}}
\end{equation}
suffices. Therefore, the probability (i.e., the sum of all three summands of Equation~\ref{eq:three-summands}) is bounded by~$3\delta$.
Replacing $\delta$ by $\nicefrac{\delta}{3}$ in Equations~\ref{eq:case2} and~\ref{eq:case1} (and in the definition of $D$) yields the desired bound on $n_0(\eps,\delta,\alpha)$.
\end{proof}

\section{Proof of Theorem~\ref{thm:envelope}}
\label{app:envelope}

\begin{proof}
Let $\eps > 0$.
Having $\Ex_\mu[F] < \infty$ implies that there is $v_0\in\mathbb{R}^+$
such that $\Ex_\mu\bigl[1_{\{V\geq v_0\}}F\bigr] < \eps$.
Since we can write
\[\Ex_\mu[f] = \Ex_\mu\bigl[f\cdot1_{\{V\leq v_0\}}\bigr]+\Ex_\mu\bigl[f\cdot1_{\{V > v_0\}}\bigr]\]
it suffices to show that almost surely there exist $n_1$ such that
\begin{equation}\label{eq:part1}
(\forall n\geq n_1)~(\forall{f\in\mathcal{F}}):~\Bigl\lvert \Ex_\mu\bigl[f\cdot1_{\{V\geq v_0\}}\bigr]-\Ex_{\mu_n}\bigl[f\cdot1_{\{V\geq v_0\}}\bigr]\Bigr\rvert\leq 2\eps
\end{equation}
and that almost surely there exist $n_2$ such that
\begin{equation}\label{eq:part2}
(\forall n\geq n_2)~(\forall{f\in\mathcal{F}}):~\Bigl\lvert \Ex_\mu\bigl[f\cdot1_{\{V < v_0\}}\bigr]-\Ex_{\mu_n}\bigl[f\cdot1_{\{V\leq v_0\}}\bigr]\Bigr\rvert\leq 3\eps.
\end{equation}

We begin by showing Equation~\eqref{eq:part1}:
the law of large numbers implies that almost surely, there exists~$n_1$
such that $\Ex_{\mu_n}\bigl[1_{\{V\geq v_0\}}F\bigr] < 2\eps$, for every $n\geq n_1$.
Since every $f\in\mathcal{F}$ satisfies $0\leq f\leq F$,
it follows that $0\leq \Ex_{\mu}\bigl[1_{\{V\geq v_0\}}f\bigr] < \eps$,
and $0\leq \Ex_{\mu_n}\bigl[1_{\{V\geq v_0\}}f\bigr] < 2\eps$ for $n\geq n_1$.
This implies Equation~\eqref{eq:part1}.

It remains to show Equation~\eqref{eq:part2}:
set $\eps' = \frac{\eps}{F(v_0)+1}$.
The Glivenko-Cantelli Theorem implies that almost surely there exists
$n_2$ such that
\[(\forall n\geq n_2)~(\forall{v\in\mathbb{R}^+}):~\bigl\lvert q(v)-q_n(v)\bigr\rvert\leq \eps'.\]
Let $f\in\mathcal{F}$. By monotonicity of $f$, it follows that there
is a sequence $0=a_0,a_1,\ldots,a_N=v_0$ such that $f$ does not
change by more than $\eps$ within each interval  $[a_i,a_{i+1})$,
(i.e., $\sup_{x,y\in[a_i,a_{i+1})}\bigl\lvert f(x)-f(y)\bigr\rvert <
\eps$). Consider the piecewise constant function
\[f_\eps=f(a_0) + \sum_{i}\bigl(f(a_{i+1})- f(a_{i})\bigr)1_{\{V\geq a_i\}}.\]
Note that $f_\eps$ gets the value $f(a_i)$ on each interval
$[a_i,a_{i+1})$. Thus, $\bigl\lvert f(v)-f_{\eps}(v)\bigr\rvert\leq \eps$ for
every $v \leq v_0$. Therefore, $\bigl\lvert\Ex_\mu[f1_{\{V < v_0\}}]
- \Ex_{\mu}[f_\eps1_{\{V < v_0\}}]\bigr\rvert \leq \eps$ and
$\bigl\lvert\Ex_{\mu_n}[f] - \Ex_{\mu_n}[f_\eps]\bigr\rvert \leq
\eps$. So, it suffices to show that
\[
(\forall n\geq n_2):~\Bigl\lvert \Ex_\mu\bigl[f_\eps1_{\{V < v_0\}}\bigr]-\Ex_{\mu_n}\bigl[f_\eps1_{\{V < v_0\}}\bigr]\Bigr\rvert\leq \eps.
\]
Indeed, for $n\geq n_2$:
\begin{align*}
\Bigl\lvert \Ex_\mu\bigl[f_\eps1_{\{V < v_0\}}\bigr]-\Ex_{\mu_n}\bigl[f_\eps1_{\{V < v_0\}}\bigr]\Bigr\rvert
&\leq
\sum_{i}\bigl(f(a_{i+1})- f(a_{i})\bigr)\cdot\bigl\lvert q(a_i) - q_n(a_i)\bigr\rvert\\
&\leq
\sum_{i}\bigl(f(a_{i+1})- f(a_{i})\bigr)\cdot\eps' \tag{by definition of $n_2$}\\
&\leq f(v_0)\cdot\eps'\\
&\leq \eps. \tag{by definition of $\eps'$}
\end{align*}

\end{proof}

\section{Proof of Corollary~\ref{c:weakGC}}
\label{app:weakGC}

\begin{proof}
We begin with the first item.
Let $\delta > 0$.
It suffices to prove that
\[\Pr \left[\forall t  : \bigl\lvert F(t) - F_n(t)\bigr\rvert \leq \eps\cdot F(t)^{\alpha(n)} \right] \leq 3\delta\]
for a large enough $n$.
To this end, we set $q,p$ so that each of the first two summands
in Lemma~\ref{lem:quant} is at most $\delta$. Specifically,
\[q  = \delta\]
and
\[p=\left(\frac{\eps\delta\ln\Bigl(\frac{1+\alpha}{2\alpha}\Bigr)}{\ln\ln\bigl(\frac{n}{\delta}\bigr)
+\ln\Bigl(\frac{1+\alpha}{2\alpha}\Bigr)}\right)^{\frac{2}{1-\alpha}}.\]
As required by the premise of Lemma~\ref{lem:quant}, $p\leq \eps^{\frac{1}{1-\alpha}}$.
(The other requirement, $n\geq\eps^{-\frac{1}{1-\alpha}}$, will be verified at the end of the proof.)

Plugging these values for $p,q$, the last summand becomes
\[ 2\exp\bigl(-2n (\eps p^\alpha)^2\bigr)
= 2\exp\left(-2n \eps^2
\left(\frac{\eps\delta\ln\Bigl(\frac{1+\alpha}{2\alpha}\Bigr)}{\ln
\ln\bigl(\frac{n}{\delta}\bigr)+\ln\Bigl(\frac{1+\alpha}{2\alpha}\Bigr)}\right)^{\frac{4\alpha}{1-\alpha}}\right).\]
We need to verify that the above expression becomes less than
$\delta$ for large $n$. Equivalently, that
\[
\lim_{n\to\infty}n\left(\frac{\eps\delta\ln\Bigl(\frac{1+\alpha}{2\alpha}\Bigr)}{\ln
\ln\bigl(\frac{n}{\delta}\bigr)+\ln\Bigl(\frac{1+\alpha}{2\alpha}\Bigr)}\right)^{\frac{4\alpha}{1-\alpha}}=
\infty.
\]
Rewriting $\alpha = 1-\beta$ gives
\[
\lim_{n\to\infty}n\left(\frac{\eps\delta\ln\Bigl(1+\frac{\beta}{2-2\beta}\Bigr)}{
\ln
\ln\bigl(\frac{n}{\delta}\bigr)+\ln\Bigl(1+\frac{\beta}{2-2\beta}\Bigr)}\right)^{\frac{4-4\beta}{\beta}}=
\infty.
\]
Since we are focusing on small value of $\beta$ we can assume that
$\beta\geq 1/2$.
Using that $\nicefrac{x}{2}\leq \ln(1+x) \leq x$ for $x\in[0,1]$ and
$\beta\leq \nicefrac{1}{2}$, it suffices that we show
\[\lim_{n\to\infty}n  \left(\frac{\eps\delta\nicefrac{\beta}{2}}{ \ln \ln(\nicefrac{n}{\delta})+1}\right)^{\frac{1}{\beta}}=\infty,\]
or, by taking ``$\ln$'', that
\[\lim_{n\to\infty}\biggl(\ln n - \frac{1}{\beta}\Bigl(\ln(\nicefrac{1}{\eps}) + \ln(\nicefrac{1}{\delta}) +
 \ln(\nicefrac{4}{\beta})+\ln\bigl(\ln \ln(\nicefrac{n}{\delta})+1\bigr)\Bigr)\biggr)=\infty.\]
To this end, it suffices that $\nicefrac{1}{\beta}\ln(\nicefrac{1}{\beta})\leq \nicefrac{\ln(n)}{2}$,
which holds for $\beta \geq c\cdot \nicefrac{\ln\ln(n)}{\ln(n)}$, where $c$ is a sufficiently large constant.

It remains to check that the condition stated in Lemma~\ref{lem:quant},
that $n\geq \eps^{-\frac{1}{1-\alpha}}= \eps^{-\frac{1}{\beta}}$,
is satisfied.
Indeed, for a sufficiently large $n$
\[\eps^{-\frac{1}{\beta}} \leq \nicefrac{1}{\eps}^{c\cdot\nicefrac{\ln(n)}{\ln\ln(n)}}=
\exp\bigl(c\ln(\nicefrac{1}{\eps})\cdot\nicefrac{\ln(n)}{\ln\ln(n)}\bigr)
< \exp\bigl(\ln(n)\bigr)=n.\]

\begin{sloppypar}
For the second item, let $Y_1\leq Y_2 \leq \cdots \leq Y_n$ 
denote the sequence obtained by sorting $X_1,X_2,\ldots,X_n$.
Note that it suffices to show that the probability that
\[Y_1 < \frac{1}{2n}\]
is at least $\nicefrac{1}{2}$:
indeed, this event implies that
\begin{align*}
F_n\left(\frac{1}{2n}\right) - F\left(\frac{1}{2n}\right) 
&\geq 
\frac{1}{n} - \frac{1}{2n}\\
&= \frac{1}{2n}\\
&\geq  \frac{1}{10}\cdot\left(\frac{1}{2n}\right)^{1-\frac{1}{2\ln n}} \tag{since $n\geq 2$}\\
&= 
\frac{1}{10}\cdot F\left(\frac{1}{2n}\right)^{1-\frac{1}{2\ln n}},
\end{align*}
which implies the conclusion with $c_2=\nicefrac{1}{2}$.
\end{sloppypar}

Thus, it remains to show that with probability of at least $\frac{1}{2}$, we have $Y_1 < \frac{1}{2n}$:
\[\Pr\left[Y_1 \geq \frac{1}{2n}\right] = \Pr\left[\forall i\leq n:~X_i \geq \frac{1}{2n}\right] = \left(1-\frac{1}{2n}\right)^n \geq \frac{1}{2}.\qedhere\]

\end{proof}

\end{document}